\documentclass[conference]{IEEEtran}
\pdfoutput=1 
\IEEEoverridecommandlockouts

\usepackage{cite}
\usepackage{amsmath,amssymb,amsfonts}
\usepackage{algorithmic}
\usepackage{algorithm}
\usepackage{graphicx}
\usepackage{booktabs}
\usepackage{textcomp}
\usepackage{xcolor}
\usepackage{subcaption} 
\usepackage{amsthm}
\usepackage{multirow}
\usepackage{makecell}
\usepackage[hidelinks,breaklinks]{hyperref}
\theoremstyle{definition}
\newtheorem{definition}{Definition}

\newtheorem{proposition}{Proposition}

\begin{document}

\title{FedTopo: Relation-Level Topology Sharing for Model-Heterogeneous Federated Learning
\thanks{$^{*}$ Corresponding author: Jing Wang and Lipo Wang}
}

\author{\IEEEauthorblockN{1\textsuperscript{st} Zhaoyang Ma}
\IEEEauthorblockA{
\textit{Beijing Jiaotong University}\\
Beijing, China \\
zhy.ma@bjtu.edu.cn}
\and
\IEEEauthorblockN{2\textsuperscript{nd} Zhihao Wu}
\IEEEauthorblockA{
\textit{Beijing Jiaotong University}\\
Beijing, China \\
zhwu@bjtu.edu.cn}
\and
\IEEEauthorblockN{3\textsuperscript{rd} Xin Gao}
\IEEEauthorblockA{
\textit{Nanyang Technological University}\\
Singapore, Singapore \\
N2505177F@e.ntu.edu.sg}
\and
\IEEEauthorblockN{4\textsuperscript{th} Lipo Wang $^{*}$}
\IEEEauthorblockA{
\textit{Nanyang Technological University}\\
Singapore, Singapore \\
ELPWang@ntu.edu.sg}
\and
\IEEEauthorblockN{5\textsuperscript{th} Youfang Lin}
\IEEEauthorblockA{\textit{Beijing Jiaotong University}\\
Beijing, China \\
yflin@bjtu.edu.cn}
\and
\IEEEauthorblockN{6\textsuperscript{th} Jing Wang $^{*}$}
\IEEEauthorblockA{\textit{Beijing Jiaotong University}\\
Beijing, China \\
wj@bjtu.edu.cn}
}

\maketitle

\begin{abstract}
Federated learning (FL) enables collaborative learning over decentralized data silos without centralizing raw data. However, heterogeneous local architectures often induce non-aligned representation spaces, making it difficult to transfer global knowledge across silos.
Existing paradigms share this knowledge as model parameters, distilled predictions, or class prototypes, yet all encode it in an absolute space that must be aligned across clients. 
Heterogeneous backbones break this alignment, so the shared knowledge becomes unreliable and misleads local training. 
We propose FedTopo, a relation-level framework that encodes global knowledge as class relation topology, capturing how classes relate within each client rather than where they lie in feature space. 
Each client builds its relation topology from local prototypes and uploads it with class statistics.  
The server then aggregates these relations in a reliability-aware manner that down-weights weakly supported ones, and broadcasts the global topology to clients. 
The global topology guides local training by emphasizing topology-similar negative classes. 
Experiments on three datasets under eight heterogeneous backbones show that FedTopo consistently outperforms parameter-, distillation-, and prototype-sharing baselines, with low communication and no inference overhead.
Our code is available at \url{https://github.com/Zhaoyang-Ma/FedTopo}.

\end{abstract}


\section{Introduction}

Federated learning (FL) has become an important paradigm for collaborative analytics over decentralized data silos, where multiple institutions, platforms, or edge devices jointly exploit distributed data without centralizing raw records \cite{mcmahan2017communication, yang2019federated, kairouz2021advances, liu2021projected}. Conventional FL assumes that all clients share an identical architecture, so the server can aggregate their parameters directly \cite{mcmahan2017communication}. In practice, however, decentralized data silos often differ in computational resources, deployment constraints, and personalization needs \cite{li2020federated, luopan2023fedknow}. Model-heterogeneous FL removes this assumption by allowing clients to keep different architectures \cite{diaoheterofl}. The central challenge then becomes transferring global knowledge across heterogeneous models and their non-aligned representation spaces \cite{li2019fedmd, tan2022fedproto}.

\begin{figure}
    \centering
    \includegraphics[width=\linewidth]{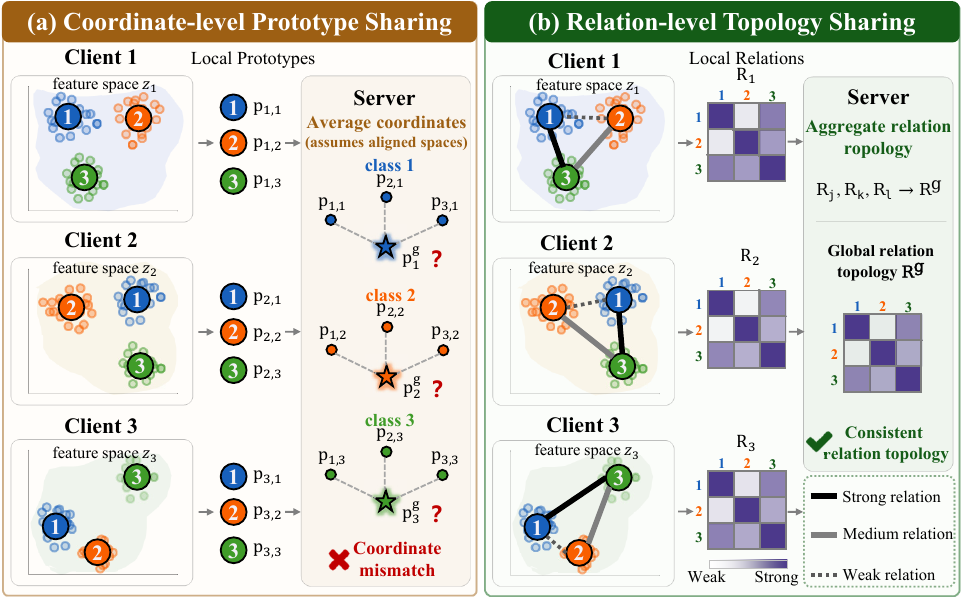}
    \caption{Motivation of FedTopo: from coordinate-level prototypes to relation-level topology. Rather than averaging prototype coordinates across non-aligned feature spaces, FedTopo aggregates intra-client class relations into a more consistent global topology.}
    \label{fig:motivation}
\end{figure}

Existing model-heterogeneous FL methods mainly address this challenge through different global knowledge-sharing paradigms \cite{zhang2025htfllib}. Partial parameter-sharing methods keep feature extractors heterogeneous while sharing lightweight homogeneous modules, such as classifier heads \cite{liang2020think, zhu2021data, chen2022bridging, yi2023fedgh, liang2025fedecover}. Knowledge distillation-based methods transfer prediction-level information, usually through public or shared auxiliary data that serves as a common input interface \cite{li2019fedmd, lin2020ensemble, wu2022communication, weng2026fedskd, li2026fedcd}. Prototype-sharing methods exchange class-wise prototypes as compact semantic summaries of each class \cite{zhang2024fedtgp, jeong2018communication, tan2022fedproto, zhang2024upload}. Despite their differences, all three paradigms encode global knowledge in an \emph{absolute} space that must be aligned across clients, whether as shared weights, a shared prediction interface, or a shared feature coordinate system. Such alignment rarely holds under model heterogeneity.

Among these paradigms, we focus on prototype sharing, which transfers knowledge at the class level without aggregating parameters or requiring auxiliary public data \cite{tan2022fedproto}. In this paradigm, each client computes class prototypes in its own feature space, and the server averages them into global prototypes. This implicitly assumes that prototypes from different clients are comparable at the coordinate level. However, under heterogeneous backbones, the same feature dimension may encode different semantic factors across clients \cite{zhang2024fedtgp}. As a result, directly averaging such prototypes may yield unreliable global prototypes and biased local guidance \cite{liu2024model}. 

The root of this unreliability lies in comparing \emph{absolute} coordinates across clients. A natural alternative is to share how classes relate to one another within each client, rather than where each class lies. We call this relative inter-class structure \emph{relation-level topology}, whose distinction from coordinate-level prototypes is illustrated in Fig.~\ref{fig:motivation}. Even when absolute representations are inconsistent across clients, semantically related classes tend to preserve similar neighborhood structures. Relation topology can therefore serve as a weaker yet more transferable knowledge carrier than absolute prototypes.

In this paper, we propose \textbf{Fed}erated \textbf{Topo}logy-guided Learning (\textbf{FedTopo}), a relation-level knowledge sharing framework for model-heterogeneous FL. FedTopo represents global knowledge as class-level relation topology, capturing the relative structure among classes rather than their absolute positions. Each client first derives local class relations from its prototypes, and the server aggregates the valid ones into a global topology. Because each client observes only a few class pairs under non-IID data, FedTopo also estimates a relation reliability matrix that down-weights weakly supported relations during aggregation, yielding a more robust global topology. The resulting topology then guides local training by emphasizing topology-similar negative classes, transferring global knowledge without any coordinate-level alignment. We evaluate FedTopo on CIFAR-10, CIFAR-100, and Tiny-ImageNet under both Dirichlet and pathological non-IID settings with eight heterogeneous backbones. It consistently improves over parameter-, distillation-, and prototype-sharing baselines while incurring low communication and no inference overhead. 

The main contributions of this work are summarized as follows:
\begin{itemize}
    \item We propose \textbf{FedTopo}, a model-heterogeneous FL framework that shares relation-level class topology rather than model parameters or prototype coordinates.
    \item We develop reliability-aware topology aggregation that suppresses unreliable relations, and topology-guided local training that sharpens class discrimination.
    \item Experiments across three datasets and eight heterogeneous backbones show that FedTopo consistently outperforms parameter-, distillation-, and prototype-sharing baselines, with low communication and no inference overhead.
\end{itemize}

\begin{figure*}
    \centering
    \includegraphics[width=\linewidth]{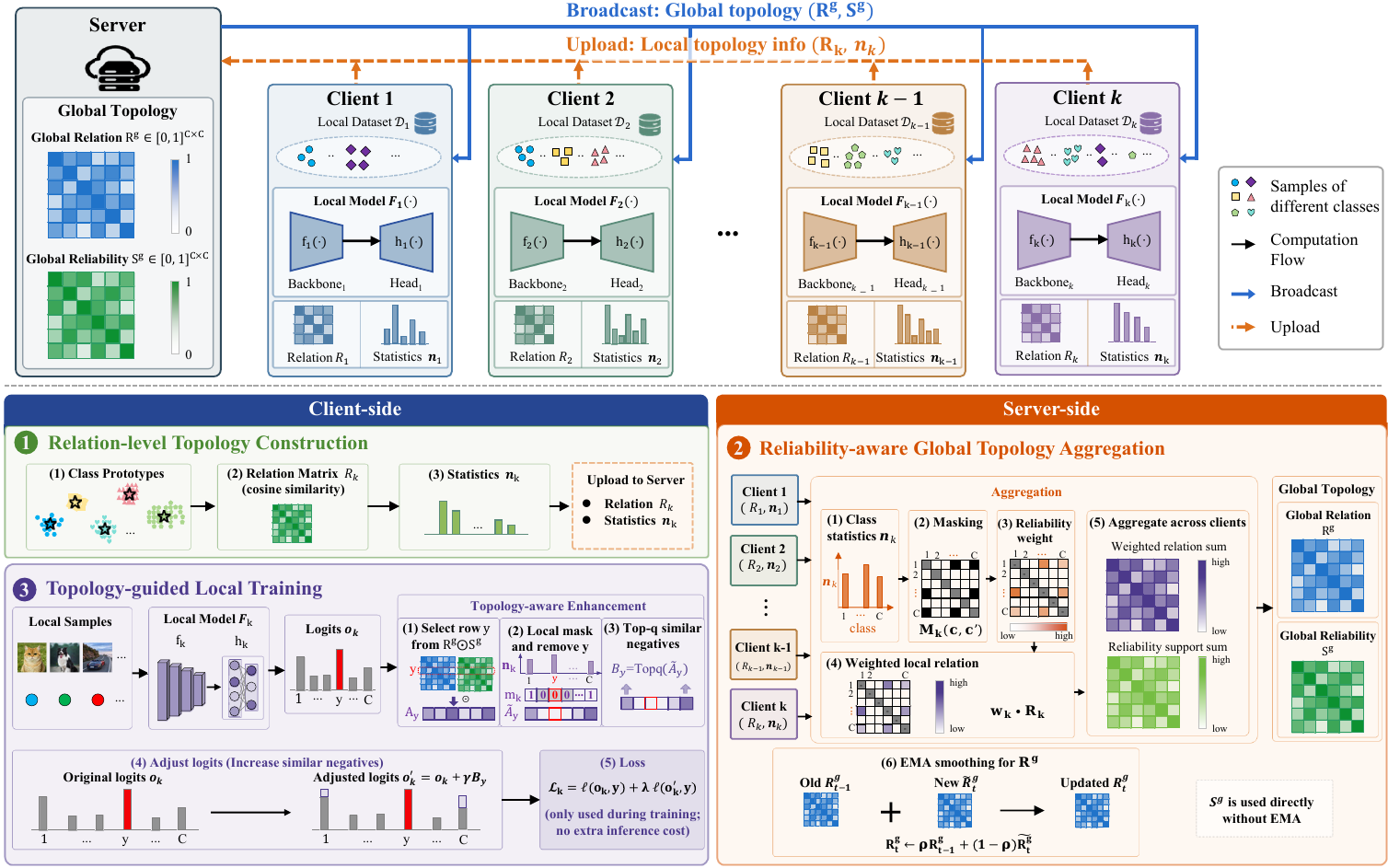}
    \caption{Framework of FedTopo. Each client constructs and uploads its local relation topology $R_k$ and class statistics $\mathbf{n}_k$. The server aggregates them into a global topology $R^g$ and a reliability matrix $S^g$, which are broadcast back to clients for topology-guided local training without coordinate-level feature alignment.}
    \label{fig:main}
\end{figure*}

\section{Related Work}
\label{sec:related}

\textbf{Model-heterogeneous Federated Learning.} 
Conventional FL aggregates parameters across clients and therefore assumes an identical architecture, which is restrictive when clients differ in computational resources, data distributions, and personalization needs \cite{mcmahan2017communication, li2020federated, kairouz2021advances}. Model-heterogeneous FL removes this assumption by allowing clients to keep different architectures \cite{diaoheterofl, yuan2024hetefedrec}, turning the central problem into transferring global knowledge across non-aligned local models \cite{zhang2025htfllib}. Existing studies enable this collaboration through three knowledge-sharing paradigms, namely partial parameter sharing \cite{liang2020think, collins2021exploiting}, prototype sharing \cite{tan2022fedproto, zhang2024fedtgp}, and knowledge distillation-based sharing \cite{shen2020federated, wu2022communication}.

\textbf{Partial parameter sharing.}
Partial parameter sharing methods keep local backbones heterogeneous while introducing lightweight homogeneous modules to carry global knowledge \cite{diaoheterofl}. One common design decouples each local network into a heterogeneous backbone and a shared classifier head that is aggregated on the server. LG-FedAvg \cite{liang2020think} aggregates these homogeneous classifier heads on the server. FedGen \cite{zhu2021data} adds a server-side generator that produces generalized features for classifier alignment, whereas FedGH \cite{yi2023fedgh} refines the global head from aggregated class-wise representations. FedRoD \cite{chen2022bridging} extends this idea with a dual-head design that jointly optimizes a generic global head and a personalized local head to counter data skew. FedSCE \cite{zhang2025subspace} instead targets the aggregation rather than the architecture, constraining each client's layer updates to a dynamically maintained local subspace and weighting contributions by their feature- and parameter-space update distances. A complementary line treats the global parameter space as a master network to fit heterogeneous hardware budgets. For instance, FedEcover \cite{liang2025fedecover} extracts submodels by sampling weight subsets without replacement, letting resource-constrained devices contribute partial updates. Overall, this paradigm collaborates by confining weight synchronization to shared parameter segments or lightweight homogeneous modules.

\textbf{Knowledge distillation-based sharing.}
Knowledge distillation-based methods transfer knowledge at the prediction level, usually through a public or shared auxiliary dataset that serves as a common input space \cite{li2019fedmd, lin2020ensemble}. These methods typically rely on auxiliary models or mutual distillation rather than a shared dataset alone. FML \cite{shen2020federated} shares output logits to co-train the auxiliary and local models, while FedKD \cite{wu2022communication} applies adaptive mutual distillation with a lightweight student to lower communication cost. FedMRL \cite{yi2024federated} adds a small shared auxiliary model and fuses auxiliary and local representations for personalized inference. More recent methods extend this paradigm further, with FedSKD \cite{weng2026fedskd} using similarity-based transfer and FedCD \cite{li2026fedcd} combining feature- and logit-level distillation. In short, this paradigm conveys global knowledge through shared prediction interfaces, auxiliary data, or auxiliary models.

\textbf{Prototype sharing.}
Prototype sharing methods transmit lightweight class-wise prototypes as global semantic knowledge \cite{tan2022fedproto}. Local prototypes are extracted from client representations, aggregated into global prototypes on the server, and then used to guide local training. FD \cite{jeong2018communication} applies prototype guidance in the logit space, whereas FedProto \cite{tan2022fedproto} and FedTGP \cite{zhang2024fedtgp} build prototypes in intermediate feature spaces for representation-level collaboration. FedTGP additionally improves prototype quality by adaptively enlarging inter-class discriminability among global prototypes. FedKTL \cite{zhang2024upload} instead uses a server-side generator to synthesize prototype-conditioned images for local training, though it is designed mainly for image tasks. FedRE \cite{yao2026fedre} entangles each client's prototypes and their one-hot labels with shared random weights into a single representation, which the server uses to train a global classifier. Among these paradigms, both parameter sharing and prototype sharing encode knowledge in an absolute coordinate space, whose cross-client consistency rarely holds under model heterogeneity. These limitations motivate the relation-level topology sharing studied in this work.

\section{Problem Statement and Motivation}
\subsection{Heterogeneous Federated Learning Setting}

We consider an FL setting with $N$ clients. Each client $k$ holds a private dataset $\mathcal{D}_k = \left\{ (x_i, y_i) \right\}_{i=1}^{n_k}$ sampled from a client-specific distribution $\mathcal{P}_k$. Let $n_{k,c}=\left|{(x,y)\in\mathcal{D}_k \mid y=c}\right|$ denote the number of class-$c$ samples on client $k$. 

Unlike homogeneous FL, each client here may use its own model architecture. The local model of client $k$ comprises a backbone $f_k(\,\cdot\,; \theta_k^b)$ and a classifier head $h_k(\,\cdot\,; \theta_k^h)$. Given an input $x$, client $k$ first extracts the representation $z_k = f_k(x; \theta_k^b)$, and then produces the prediction logits $o_k = h_k(z_k; \theta_k^h)$. Since $f_k$ varies across clients, the representations $z_k$ may lie in client-specific feature spaces with inconsistent coordinate semantics. 

For each client $k$, the local empirical risk is defined as
\begin{equation}
F_k(\theta_k)= \frac{1}{n_k} \sum_{(x,y)\in \mathcal{D}_k} \ell\big(h_k\big(f_k(x;\theta_k^b);\theta_k^h\big),y\big),
\end{equation}
where $\theta_k={\theta_k^b,\theta_k^h}$, $n_k=|\mathcal{D}_k|$, and $\ell(\cdot,\cdot)$ is the cross-entropy loss. Our goal is to enhance these heterogeneous local models ${\theta_k}_{k=1}^N$ through collaborative knowledge sharing, rather than to learn a single global model or average parameters directly. 

However, independently minimizing each $F_k$ may lead to biased local decision boundaries under non-IID data. Therefore, the key challenge is to exploit global knowledge across heterogeneous clients without direct parameter aggregation or coordinate-level feature alignment.

\subsection{Motivation: From Absolute to Relational Knowledge Sharing}

Exploiting global knowledge across heterogeneous clients first requires deciding what form of knowledge to share. Among the three paradigms in Section~\ref{sec:related}, knowledge distillation needs a shared public or auxiliary dataset as a common interface \cite{li2019fedmd}. Such data is often unavailable under strict privacy constraints, so we target the data-free setting, where knowledge is shared as model parameters or class prototypes. 

Both data-free options, however, encode global knowledge in an \emph{absolute} space that must be aligned across clients. Parameter sharing presumes matched architectures and aligned weights, while prototype sharing presumes comparable feature coordinates. Neither holds reliably under heterogeneous backbones, so directly aggregating parameters or prototypes yields misleading guidance for local training. 

These assumptions motivate a weaker, \emph{relational} form of global knowledge that avoids comparing absolute coordinates across clients. Instead of sharing absolute parameters or prototype coordinates, we share relative class topology, that is, how classes relate within each client's local feature space. Because this depends only on intra-client geometry, it relaxes the need for cross-client representation consistency.

\section{Proposed Framework}

\subsection{Framework Overview}

Motivated by the cross-client coordinate mismatch analyzed above, FedTopo transfers global knowledge as relation-level class topology rather than as absolute feature representations. FedTopo therefore models how classes relate within each client and assembles these local relations into a global class topology, without aggregating heterogeneous parameters or absolute prototypes.

As illustrated in Fig.~\ref{fig:main}, FedTopo comprises three collaborative stages, namely \textbf{Relation-level Topology Construction}, \textbf{Reliability-aware Global Topology Aggregation}, and \textbf{Topology-guided Local Training}. Algorithm~\ref{alg:fedtopo} summarizes
the overall procedure.

At the beginning of each communication round, every client trains its own heterogeneous local model on its private data. Instead of sharing model parameters or feature representations, each client builds a local relation topology $R_k$ that captures the semantic relationships among classes within its own feature space. It simultaneously records the local class statistics $\mathbf{n}_k = \big(n_{k,c}\big)_c$, i.e., the per-class sample counts that quantify the statistical support of these relationships. Only the relation-level topology information $(R_k,\mathbf{n}_k)$ is uploaded to the server.

The server aggregates the uploaded relations $(R_k,\mathbf{n}_k)$ in a reliability-aware manner into a global topology $(R^g,S^g)$, where $R^g$ is the global relation matrix and $S^g$ its reliability matrix. During aggregation, weakly observed or unreliable class relations are down-weighted to improve the robustness of the global estimate. The server then broadcasts $(R^g,S^g)$ to the participating clients.

On receiving $(R^g,S^g)$, each client performs topology-guided local training, in which topology-similar negative classes are selectively emphasized to sharpen local discrimination. The global topology is used only during training, so it adds no inference overhead at test time.

\subsection{Relation-level Topology Construction}

In the first stage, each client $k$ extracts class prototypes from its own representation space. Given the local backbone $f_k(\cdot;\theta_k^b)$, an input sample $x$ is mapped to a representation $z_k=f_k(x;\theta_k^b)$. The prototype of class $c$ is the mean of its sample representations
\begin{equation}
p_{k,c}=\frac{1}{n_{k,c}}\sum_{(x,y)\in\mathcal{D}_k,\,y=c} z_k,
\label{eq:prototype}
\end{equation}
where $n_{k,c}$ is the number of class-$c$ samples on client $k$. When $n_{k,c}=0$, class $c$ is absent from client $k$ and produces no prototype. 

From these prototypes, FedTopo constructs a local relation topology $R_k\in[0,1]^{C\times C}$ that encodes the pairwise semantic relationships among classes. For any two classes $c$ and $c'$ that are both present on client $k$, we measure the cosine similarity of their prototypes and rescale it from $[-1, 1]$ to $[0,1]$
\begin{equation}
R_k(c,c') =
\frac{1+\cos(p_{k,c},p_{k,c'})}{2}
=
\frac{1}{2}
\left(
1+
\frac{p_{k,c}^{\top}p_{k,c'}}
{\|p_{k,c}\|\|p_{k,c'}\|}
\right).
\label{eq:local_relation}
\end{equation}

A larger $R_k(c,c')$ thus indicates a stronger semantic relation between the two classes in client $k$'s feature space.  

A relation requires both prototypes, so $R_k(c,c')$ is defined only when $n_{k,c}>0$ and $n_{k,c'}>0$. Otherwise the entry is masked out and excluded from the server-side aggregation in Section~\ref{sec:agg}.

Importantly, the prototypes $p_{k,c}$ never leave their client. They are used only to derive the intra-client relations $R_k$, and are neither uploaded nor compared across clients. As a result, $R_k$ encodes \emph{relative} class geometry rather than \emph{absolute} feature coordinates, which is precisely what makes relation-level sharing viable under backbone heterogeneity. FedTopo thus requires no common cross-client coordinate system, only that each client's local class geometry preserves transferable semantic structure.

\subsection{Reliability-aware Global Topology Aggregation}
\label{sec:agg}

Each participating client uploads its local relation matrix $R_k$ together with the per-class sample counts $\mathbf{n}_k$. From these, the server constructs a global relation topology $R^g$ that summarizes class-level semantic relationships across decentralized data. These statistics let the server gauge how reliable each local relation is. Compared with methods that transmit parameters, logits, features, or class prototypes, FedTopo shares only relation-level topology and coarse class-count statistics. Since feature representations and prototype coordinates remain local, it reduces direct exposure of client-specific representation information, while the count statistics can be replaced by noisy counts in privacy-sensitive deployments. 

A naive solution would average the local matrices $\{R_k\}$ entry by entry, but this is unreliable under non-IID data for two reasons. First, clients observe different subsets of classes, so a class pair that is absent on a client produces no valid relation at all. Second, even an observed relation can be statistically unstable when the two classes are supported by only a few local samples. FedTopo addresses both with a binary mask that discards unobserved class pairs and a reliability weight that down-weights weakly supported ones. 

For client $k$, the class-pair observation mask
$M_k\in\{0,1\}^{C\times C}$ is defined as 
\begin{equation}
M_k(c,c') =
\mathbb{I}[n_{k,c}>0]\cdot \mathbb{I}[n_{k,c'}>0],
\label{eq:observation_mask}
\end{equation}
where $M_k(c,c')=1$ indicates that both classes are observed on client $k$, and
only such pairs are eligible for aggregation.

The mask alone does not solve the second issue. An observed relation can still be unreliable when the two classes are weakly supported, so FedTopo attaches a pair-wise reliability weight to each surviving relation
\begin{equation}
w_k(c,c') =
M_k(c,c')\cdot
\frac{n_{k,c}n_{k,c'}}
{n_{k,c}+n_{k,c'}+\epsilon},
\label{eq:pair_weight}
\end{equation}
where $\epsilon$ is a small constant for numerical stability. The factor $\frac{n_{k,c}n_{k,c'}}{n_{k,c}+n_{k,c'}+\epsilon}$ is proportional to the harmonic mean of the two class counts. It grows large only when \emph{both} classes are well represented, and collapses toward zero as soon as either class is scarce. A weakly supported relation therefore contributes little to the global topology.

With these weights, the per-round global relation for each class pair is the reliability-weighted average of its valid local relations
\begin{equation}
\widetilde{R}^g_t(c,c') =
\frac{\sum_{k\in\mathcal{P}^{(t)}} w_k(c,c')\,R_k(c,c')}
     {\sum_{k\in\mathcal{P}^{(t)}} w_k(c,c')+\epsilon},
\label{eq:reliability}
\end{equation}
where $\mathcal{P}^{(t)}$ is the set of clients participating in round $t$. If no participating client provides a valid observation for a class pair, its global relation is set to zero.

The weights $w_k$ down-weight unreliable relations \emph{within} the aggregation, but they do not record how much total support a global relation received \emph{across} clients. FedTopo captures this with a relation reliability matrix $S^g\in[0,1]^{C\times C}$
\begin{equation}
S^g(c,c') =
\frac{
\sum_{k\in\mathcal{P}^{(t)}} w_k(c,c')
}{
\max_{a,b}\sum_{k\in\mathcal{P}^{(t)}} w_k(a,b)+\epsilon
}.
\label{eq:global_reliability}
\end{equation}

The denominator normalizes by the most strongly supported class pair, so $S^g(c,c')$ measures the relative amount of reliable evidence behind each global relation. Different class pairs can be backed by very different numbers of reliable observations. Without it, every global relation would steer local training with equal force, even one estimated from a single sparse client. $S^g$ therefore lets FedTopo modulate the influence of each relation by how well it is supported. 

Finally, to reduce round-to-round fluctuation, FedTopo smooths the global topology with an exponential moving average. Given the aggregated topology $\widetilde{R}^g_t$ from Eq.~\ref{eq:reliability}, the broadcast topology is
\begin{equation}
R^g_t \leftarrow \rho\, R^g_{t-1} + (1-\rho)\,\widetilde{R}^g_t,
\label{eq:ema}
\end{equation}
where $\rho\in[0,1)$ controls the smoothing strength. For class pairs without any valid observation in round $t$, the previous value $R^g_{t-1}$ is retained so that transiently unobserved relations are not erased.

\begin{algorithm}[t]
\caption{FedTopo Framework.}
\label{alg:fedtopo}
\begin{algorithmic}[1]
\item[] \textbf{Input:} \(N\) clients with heterogeneous local models \(\{F_k=h_k\circ f_k\}_{k=1}^{N}\), datasets \(\{\mathcal{D}_k\}\), \(C\) classes, rounds \(T\), local epochs \(E\), hyperparameters \(\lambda,\gamma,q,\rho\).
\item[] \textbf{Output:} personalized local models \(\{F_k\}_{k=1}^{N}\). 
\item[] \textbf{Initialize:} \(R^g_0=I_C\) as the default self-relation prior and \(S^g_0=\mathbf{0}\) as zero initial reliability.

\FOR{round \(t=1,\ldots,T\)}
    \STATE Server samples clients \(\mathcal{P}^{(t)}\) and broadcasts \((R^g_{t-1},S^g_{t-1})\).
    \item[] \hspace{\algorithmicindent}\textit{// Client update}
    \FOR{each client \(k\in\mathcal{P}^{(t)}\) in parallel}
        \STATE Train \(F_k\) for \(E\) epochs with the topology-guided loss (Eqs.~(\ref{eq:topo_strength})–(\ref{eq:topo_loss})) using \((R^g_{t-1},S^g_{t-1})\).
        \STATE Build prototypes and relation topology \(R_k\) (Eqs.~(\ref{eq:prototype}),~(\ref{eq:local_relation})); record class statistics \(\mathbf{n}_k\).
        \STATE Upload \((R_k,\mathbf{n}_k)\).
    \ENDFOR
    \item[] \hspace{\algorithmicindent}\textit{// Server update}
    \STATE Compute masks \(M_k\) and reliability weights \(w_k\) (Eqs.~(\ref{eq:observation_mask}),~(\ref{eq:pair_weight})).
    \STATE Aggregate into \(\widetilde{R}^g_t\) (Eq.~(\ref{eq:reliability})) and estimate \(S^g_t\) (Eq.~(\ref{eq:global_reliability})).
    \STATE Smooth: \(R^g_t \leftarrow \rho R^g_{t-1}+(1-\rho)\widetilde{R}^g_t\) (Eq.~(\ref{eq:ema})).
\ENDFOR
\RETURN \(\{F_k\}_{k=1}^{N}\).
\end{algorithmic}
\end{algorithm}

\subsection{Topology-guided Local Training}

Client $k$ receives the broadcast global topology $R^g_t$ and its reliability matrix $S^g$, and uses them as training-time guidance rather than as a feature-alignment target. We drop the round index $t$ and write $R^g$ throughout this subsection. The goal is not to force local representations into a shared coordinate space, but to sharpen each client's ability to discriminate among topologically similar classes. Concretely, when $R^g$ indicates that a negative class lies close to the ground-truth class and $S^g$ marks this relation as reliable, FedTopo amplifies that class's competition with the ground-truth class during local training.
\begin{figure*}
    \centering
    \includegraphics[width=\linewidth]{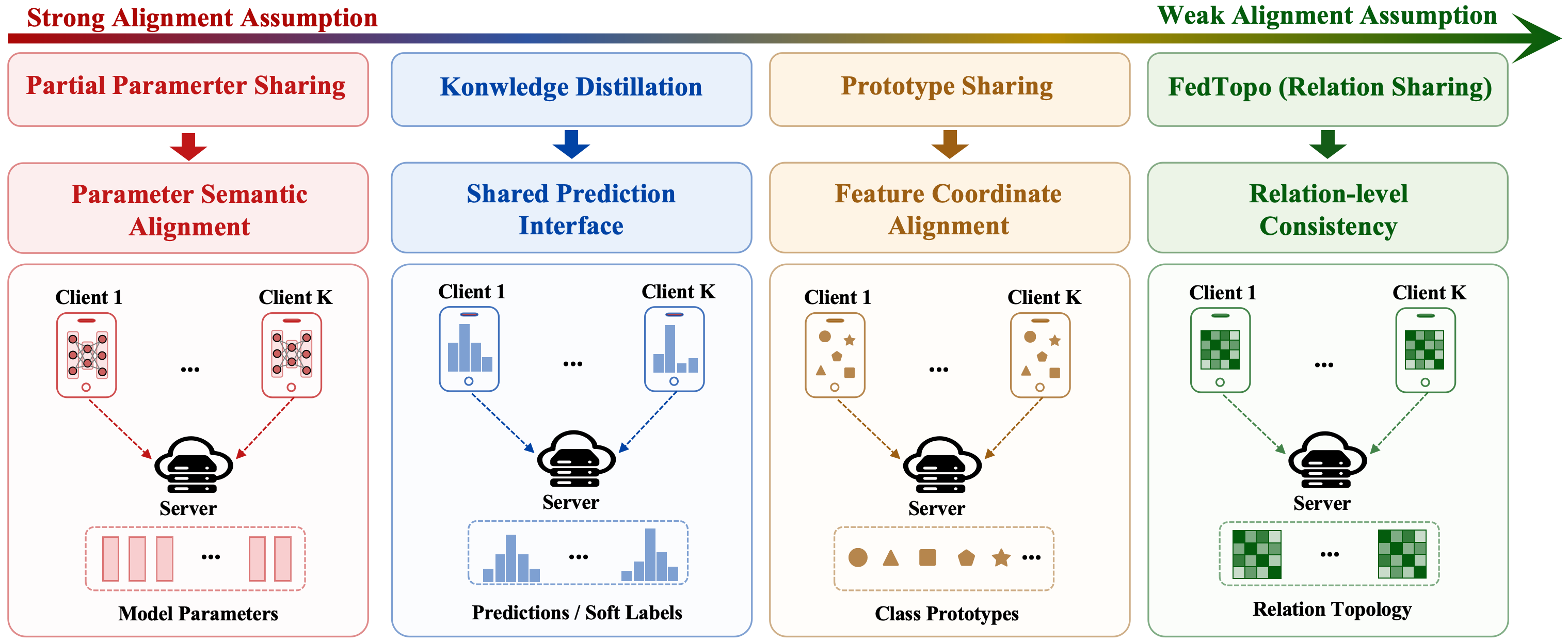}
    \caption{Comparison of global knowledge-sharing paradigms in model-heterogeneous FL. 
}
    \label{fig:paradigm}
\end{figure*}
For a training sample $(x,y)$ on client $k$, FedTopo reads out the row of the global topology indexed by the ground-truth class $y$. The raw guidance strength over candidate negative classes is
\begin{equation}
A_y = R^g(y,:) \odot S^g(y,:),
\label{eq:topo_strength}
\end{equation}
where $\odot$ denotes element-wise multiplication. The factor $R^g(y,:)$ indicates which classes are topologically similar to $y$, while $S^g(y,:)$ down-weights similarities that rest on weak global evidence, so a negative class receives a high score only when it is both similar and reliably so. 

Because client $k$ may hold only a subset of classes, the guidance must be restricted to negative classes that actually exist locally. Let $1_k(c) = \mathbb{I}\left[n_{k,c} > 0\right]$ indicate whether class $c$ is present on client $k$. The guidance strength is then masked as
\begin{equation}
\widetilde{A}_y(c)
=
A_y(c)\cdot 1_k(c)\cdot \mathbb{I}[c\neq y],
\end{equation}
where the indicator $1_k(c)$ removes locally absent classes and $\mathbb{I}[c\neq y]$ removes the ground-truth class itself.

A class may still be weakly related to many others, so applying every surviving relation would inject dense and partly noisy feedback. FedTopo therefore retains only the strongest entries of \(\widetilde{A}_y\). Specifically, a top-\(q\) operator keeps the $q$ largest entries
\begin{equation}
B_y = \operatorname{Top}_q(\widetilde{A}_y).
\label{eq:topq}
\end{equation}
The operator thus focuses the auxiliary signal on the few negatives that are both most confusable with $y$ and most reliably related to it.

Let $o_k=h_k(f_k(x;\theta_k^b);\theta_k^h)$ be the original logits, where $f_k$ is the backbone and $h_k$ the local classification head. FedTopo then forms topology-modulated logits by adding the selected guidance to the logits of the chosen negatives
\begin{equation}
o'_k = o_k + \gamma B_y,
\label{eq:modulated_logits}
\end{equation}
where $\gamma>0$ controls the modulation strength. The local objective combines the standard loss with a topology-aware auxiliary term
\begin{equation}
\mathcal{L}_k
=
\ell(o_k,y)
+
\lambda \ell(o'_k,y),
\label{eq:topo_loss}
\end{equation}
where $\ell$ is the cross-entropy loss and $\lambda > 0$ balances the standard term $\ell(o_k,y)$ against the auxiliary term $\ell(o'_k,y)$.

Raising the logits of the selected negatives makes them stronger competitors of the ground-truth class, so minimizing $\ell(o'_k,y)$ forces the backbone to pull these confusable classes further apart, yielding sharper local decision boundaries. At inference time, FedTopo uses only the original logits $o_k$. Since the global topology and labels are needed solely during training, the modulation adds no deployment-time cost.

\section{Comparative Analysis}
\label{sec:analysis}

\subsection{Comparison of Global Knowledge Sharing Paradigms}

Existing model-heterogeneous FL methods differ mainly in the form of global knowledge they exchange. Sharing such knowledge requires some consistency across clients, which we call the \emph{cross-client alignment assumption}. 
\begin{definition}[Cross-Client Alignment Assumption]
\label{def:align}
The \emph{cross-client alignment assumption} of a global knowledge-sharing paradigm refers to the required consistency among client-side knowledge carriers before they can be aggregated, compared, or transferred across clients.
\end{definition}

Under Definition~\ref{def:align}, different paradigms impose this consistency on different knowledge carriers, such as model parameters, output predictions, feature representations, or relation structures. As illustrated in Fig.~\ref{fig:paradigm}, the required alignment gradually weakens from parameter sharing, through knowledge distillation and prototype sharing, to relation-level topology sharing.

\textbf{Partial parameter-sharing methods} directly exchange model parameters or partial network weights among clients. They impose the strongest assumption, since parameter aggregation is meaningful only when local models share compatible architectures and aligned parameter spaces. This assumption becomes difficult to satisfy when clients employ heterogeneous backbones or classifier heads.

\textbf{Knowledge-distillation-based methods} avoid direct parameter aggregation by transferring prediction-level information. They relax parameter-space alignment, but still require heterogeneous models to produce comparable logits or soft predictions over the same output space. Therefore, their alignment assumption shifts from parameter consistency to functional consistency, which is often supported by public or auxiliary data.

\textbf{Prototype-sharing methods} further relax the alignment requirement by exchanging class-level feature representations. This paradigm is more architecture-agnostic because prototypes can be extracted from heterogeneous local models. However, prototype aggregation still assumes that class-level representations from different clients are comparable in a shared feature coordinate space. Under heterogeneous backbones, such feature-space comparability may be violated, leading to semantically misaligned prototype coordinates.

\textbf{FedTopo} instead transfers relation-level topology rather than absolute feature coordinates. It models the relative relations among classes within each local feature space, without comparing prototype coordinates across clients. Therefore, the required consistency shifts from cross-client coordinate alignment to intra-client relation consistency. Such a weaker assumption yields a more transferable form of global knowledge under model heterogeneity. 

\begin{figure*}[t] 
    \centering
    \begin{subfigure}[b]{0.27\textwidth}
        \centering
        \includegraphics[width=\linewidth]{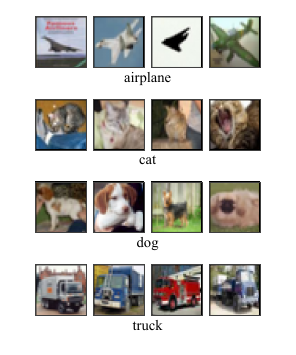}
        \caption{Real samples.}
        \label{fig:mec_5}
    \end{subfigure}
    \hfill 
    \begin{subfigure}[b]{0.44\textwidth}
        \centering
        \includegraphics[width=\linewidth]{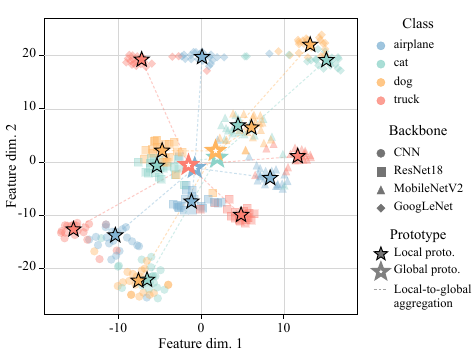}
        \caption{Prototype-coordinate mismatch.}
        \label{fig:mec_6}
    \end{subfigure}
    \hfill
    \begin{subfigure}[b]{0.27\textwidth}
        \centering
        \includegraphics[width=\linewidth]{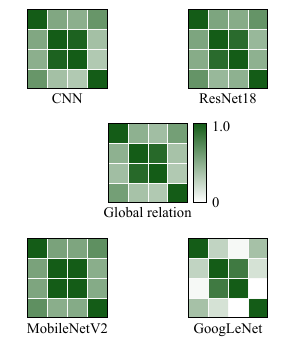}
        \caption{Relation topology consistency.}
        \label{fig:mec_7}
    \end{subfigure}
    \hfill
    \caption{Empirical mechanism visualization under model heterogeneity.}
    \label{fig:real_vis}
\end{figure*}

\begin{figure*}[t] 
    \centering
    \begin{subfigure}[b]{0.24\textwidth}
        \centering
        \includegraphics[width=1.055\linewidth]{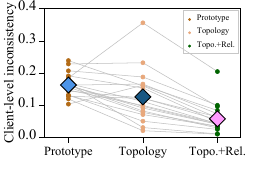}
        \caption{Observed-class inconsistency.}
        \label{fig:mec_1}
    \end{subfigure}
    \hfill 
    \begin{subfigure}[b]{0.24\textwidth}
        \centering
        \includegraphics[width=1.085\linewidth]{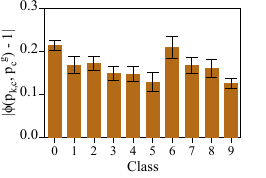}
        \caption{Prototype inconsistency.}
        \label{fig:mec_2}
    \end{subfigure}
    \hfill
    \begin{subfigure}[b]{0.25\textwidth}
        \centering
        \includegraphics[width=1.1\linewidth]{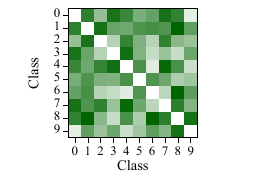}
        \caption{Topology inconsistency.}
        \label{fig:mec_3}
    \end{subfigure}
    \hfill
    \begin{subfigure}[b]{0.25\textwidth}
        \centering
        \includegraphics[width=1.07\linewidth]{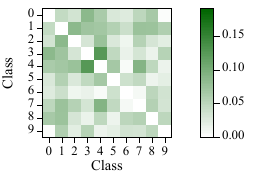}
        \caption{Reliability-modulated.}
        \label{fig:mec_4}
    \end{subfigure}
    
    \caption{Empirical mechanism visualization of prototype-coordinate and relation-topology consistency under model heterogeneity.}
    \label{fig:mechanism}
\end{figure*}

\subsection{Analytical Comparison of Prototype and Topology Sharing}
\label{sec:theory}

We analyze why relation-level topology can serve as a more robust global knowledge carrier than prototype coordinates under model heterogeneity. Our focus is not on proving convergence of heterogeneous local models, but on comparing how different shared objects are affected by client-specific feature spaces. 

\textbf{Prototype aggregation inherits the coordinate deviation.}
Prototype-sharing methods construct a global prototype by averaging local prototypes
\begin{equation}
p_c^g=\sum_{k\in\mathcal{K}_c}\beta_{k,c}\,p_{k,c},
\qquad
\sum_{k\in\mathcal{K}_c}\beta_{k,c}=1 .
\end{equation}

We decompose each local prototype as
\begin{equation}
p_{k,c}=q_{k,c}+\Delta_{k,c},
\end{equation}
where \(q_{k,c}\) denotes the latent semantic prototype of class \(c\) on client \(k\), and \(\Delta_{k,c}\) denotes the coordinate deviation induced by client \(k\)'s feature space. Let $q_c^g=\sum_{k\in\mathcal{K}_c}\beta_{k,c}q_{k,c}$ be the corresponding ideal aggregated semantic prototype. Then,
\begin{equation}
p_c^g-q_c^g
=
\sum_{k\in\mathcal{K}_c}\beta_{k,c}\Delta_{k,c},
\end{equation}
and by the triangle inequality,
\begin{equation}
\big\|p_c^g-q_c^g\big\|
\le
\sum_{k\in\mathcal{K}_c}\beta_{k,c}\big\|\Delta_{k,c}\big\|.
\end{equation}

This shows that prototype aggregation is directly exposed to the full per-client coordinate deviation. Such deviation is determined by heterogeneous feature spaces and cannot generally be removed by coordinate averaging alone.

\textbf{Relations remove a structured component of coordinate deviation.}
Instead of aggregating prototype coordinates, FedTopo aggregates the bounded cosine relation \(\phi(p_{k,c},p_{k,c'})\) defined in Eq.~\ref{eq:local_relation}. To clarify the benefit of this choice, we model a structured component of the feature-space gap as a client-specific linear transform
\begin{equation}
p_{k,c}=T_kq_{k,c},
\qquad
T_k=s_kQ_k+E_k,
\qquad
s_k>0,\quad Q_k^\top Q_k=I,
\end{equation}
where \(s_kQ_k\) denotes an isotropic scaling and orthogonal transform, and \(E_k\) denotes the remaining non-conformal residual.

\begin{proposition}[Conformal invariance of relations]
\label{prop:invariance}
If the residual vanishes, i.e., \(E_k=0\), then for every class pair \((c,c')\),
\begin{equation}
\phi(p_{k,c},p_{k,c'})
=
\phi(q_{k,c},q_{k,c'}),
\end{equation}
regardless of the scaling factor \(s_k\) and the orthogonal transform \(Q_k\).
\end{proposition}
\begin{proof}
For any vectors $a$ and $b$, we have $\langle s_kQ_ka,\,s_kQ_kb\rangle=s_k^2\langle a,b\rangle$ and $\|s_kQ_ka\|=s_k\|a\|$, hence $\cos(s_kQ_ka,\,s_kQ_kb)=\cos(a,b)$. Since $\phi$ depends on its inputs only through this cosine, its value is unchanged.
\end{proof}

Proposition~\ref{prop:invariance} shows that cosine-based relations exactly cancel the orthogonal-and-scaling component of feature-space transformations, whereas prototype coordinates change under the same component. In other words, a conformal transform may induce a large coordinate deviation \(\Delta_{k,c}=(s_kQ_k-I)q_{k,c}\) that is inherited by prototype averaging, but this component does not affect the relation value. Real heterogeneous backbones may also introduce non-conformal distortions, sampling noise, and optimization errors; these residual factors can still perturb the estimated relation topology. Thus, relation sharing does not assume perfect cross-client alignment, but removes a structured source of coordinate mismatch before aggregation. 

\textbf{Reliability weighting suppresses residual relation noise.}
The residual \(E_k\), finite-sample noise, and sparsely observed class pairs may still perturb each local relation. We write
\begin{equation}
R_k(c,c')=R^\star(c,c')+\xi_k(c,c'),
\end{equation}
where \(R^\star(c,c')\) is the ideal global relation and \(\xi_k(c,c')\) is the local relation-estimation error. Aggregating with the normalized reliability weights
\begin{equation}
\alpha_k(c,c')
=
\frac{w_k(c,c')}{\sum_j w_j(c,c')+\epsilon}
\end{equation}
gives
\begin{equation}
R^g(c,c')-R^\star(c,c')
=
\sum_{k\in\mathcal{P}^{(t)}}\alpha_k(c,c')\xi_k(c,c').
\end{equation}

This weight becomes large only when both classes are observed and sufficiently represented on client \(k\). Therefore, relations estimated from few samples receive smaller weights and contribute less to the global topology. During local training, the reliability matrix \(S^g\) further scales each relation by its total support
\begin{equation}
\bar{R}^g=R^g\odot S^g,
\end{equation}
so weakly supported relations provide weaker topology guidance.

In sum, cosine relations remove the conformal coordinate mismatch inherited by prototype averaging, and reliability-aware aggregation suppresses the residual noise from non-conformal distortions and sparse observations.

\subsection{Empirical Mechanism Visualization}
The analysis in Section~\ref{sec:theory} attributes the reliability of FedTopo to two effects. On the one hand, relation-level topology avoids the cross-client coordinate mismatch that prototype aggregation inherits. On the other hand, reliability modulation suppresses relations estimated from sparsely supported class pairs. We examine both effects empirically, and find that relation topology is already more consistent than prototype coordinates, while reliability weighting acts selectively, targeting the sparsely supported pairs whose relations are noisiest. 

As an empirical counterpart to Fig.~\ref{fig:paradigm}, we first visualize backbone-induced coordinate mismatch in isolation. To remove the confounding influence of client-side data bias, we train all heterogeneous backbones on the full CIFAR-10 training set and feed them the same real samples, as shown in Fig.~\ref{fig:real_vis}(a). Fig.~\ref{fig:real_vis}(b) shows that same-class local prototypes are scattered in the projected feature space, indicating that prototype coordinates are not naturally aligned even when all models observe the same data distribution.  In contrast, Fig.~\ref{fig:real_vis}(c) shows that the class-relation topologies induced by different backbones exhibit similar structural patterns, suggesting that relation-level topology provides a more consistent knowledge carrier than absolute prototype coordinates.

Beyond this intuitive visualization, we conduct a quantitative mechanism analysis under the federated setting. Rather than measuring final task performance, these mechanism-oriented metrics test whether relation-level topology is more stable than absolute prototype coordinates, and whether reliability modulation further attenuates unreliable relations. We define three client-level inconsistency metrics, evaluated only on the observed class set of each client, since missing classes yield no valid local prototypes or class-pair relations. To keep the prototype and relation metrics directly comparable, we measure both on the same bounded similarity scale $\phi(a,b)=\tfrac{1+\cos(a,b)}{2}\in[0,1]$, so that each metric is an $L_1$ deviation on this common scale.

\begin{figure}[t] 
    \centering
    \begin{subfigure}[b]{0.24\textwidth}
        \centering
        \includegraphics[width=\linewidth]{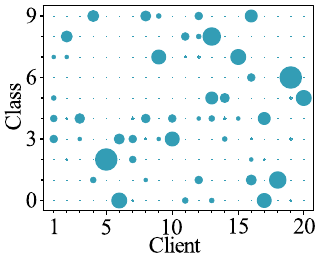}    
        \caption{$\alpha=0.1$.}
        \label{fig:data_1}
    \end{subfigure}
    \hfill 
    \begin{subfigure}[b]{0.24\textwidth}
        \centering
        \includegraphics[width=0.94\linewidth]{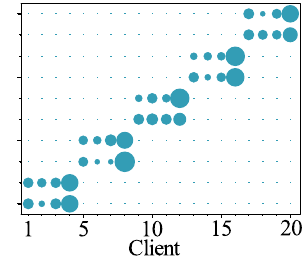}
        \caption{Pathological.}
        \label{fig:data_2}
    \end{subfigure}
    \hfill

    \caption{Visualization of non-IID data distributions on CIFAR-10. The dot size indicates the sample quantity per class.}
    \label{fig:data}
\end{figure}

For prototype coordinates, we quantify the client-level inconsistency as the deviation of each local prototype from its global prototype under perfect alignment ($\phi=1$)
\begin{equation}
D_k^{p}
=
\frac{1}{|\mathcal{C}_k|}
\sum_{c\in\mathcal{C}_k}
\big|\,\phi(p_{k,c},p_c^g)-1\,\big|,
\end{equation}
where $\mathcal{C}_k=\{c\mid n_{k,c}>0\}$ denotes the observed class set of client $k$. For relation topology, we analogously compare the local relation matrix against the global topology over the observed class pairs only
\begin{equation}
D_k^{r}
=
\frac{1}{|\Omega_k|}
\sum_{(c,c')\in\Omega_k}
\left|R_k(c,c')-R^g(c,c')\right|,
\end{equation}
where $\Omega_k=\{(c,c')\mid c,c'\in\mathcal{C}_k,\,c<c'\}$. To isolate the effect of reliability modulation, we further weight each pairwise discrepancy by $S^g(c,c')$
\begin{equation}
D_k^{rs}
=
\frac{1}{|\Omega_k|}
\sum_{(c,c')\in\Omega_k}
\left|R_k(c,c')-R^g(c,c')\right|S^g(c,c').
\end{equation}

On this common scale, comparing $D_k^{p}$ with $D_k^{r}$ tests the first effect, while comparing $D_k^{rs}$ with $D_k^{r}$ isolates the gain from reliability modulation.

\begin{table*}[t]
\centering
\caption{Performance comparison under Dirichlet and pathological non-IID settings.}
\label{tab:main_results}
\renewcommand{\arraystretch}{1.15}
\resizebox{\textwidth}{!}{
\begin{tabular}{cccccccc}
\toprule
\multirow{2}{*}{\textbf{Type}} & \multirow{2}{*}{\textbf{Method}}
& \multicolumn{3}{c}{\textbf{Dirichlet Setting}}
& \multicolumn{3}{c}{\textbf{Pathological Setting}} \\
\cmidrule(lr){3-5} \cmidrule(lr){6-8}
& & \textbf{CIFAR-10} & \textbf{CIFAR-100} & \textbf{Tiny-ImageNet}
& \textbf{CIFAR-10} & \textbf{CIFAR-100} & \textbf{Tiny-ImageNet} \\
\midrule

\textit{Reference} & \textit{Local} & $85.43\pm0.39$ & $39.93\pm0.23$ & $23.73\pm0.19$
& $83.38\pm0.33$ & $53.29\pm0.30$ & $31.47\pm0.27$ \\
\midrule

\multirow{4}{*}{\makecell[c]{Parameter\\sharing}}
& LG-FedAvg \cite{liang2020think} & $84.59\pm0.44$ & $38.49\pm0.20$ & $23.22\pm0.12$
& $84.03\pm0.36$ & $55.64\pm0.57$ & $31.69\pm0.25$ \\
& FedGen \cite{zhu2021data} & $84.45\pm0.52$ & $37.67\pm0.17$ & $23.59\pm0.43$
& $83.59\pm0.57$ & $55.62\pm0.72$ & $31.65\pm0.62$ \\
& FedGH \cite{yi2023fedgh} & $84.19\pm0.37$ & $37.90\pm0.39$ & $23.50\pm0.51$
& $84.23\pm0.45$ & $54.17\pm0.35$ & $30.89\pm0.43$ \\
& FedSCE \cite{zhang2025subspace} & $86.18\pm0.55$ & $\underline{42.01\pm0.35}$ & $\underline{25.94\pm0.43}$
& $\underline{85.84\pm0.41}$ & $55.56\pm0.46$ & $\underline{32.80\pm0.24}$ \\

\midrule
\multirow{3}{*}{\makecell[c]{Knowledge\\distillation}}
& FML \cite{shen2020federated} & $\underline{86.44\pm0.46}$ & $41.83\pm0.33$ & $23.91\pm0.38$
& $84.81\pm0.32$ & $55.05\pm0.44$ & $31.60\pm0.34$ \\
& FedKD \cite{wu2022communication} & $86.07\pm0.41$ & $41.45\pm0.45$ & $24.45\pm0.37$
& $83.68\pm0.37$ & $55.43\pm0.35$ & $31.52\pm0.28$ \\
& FedMRL \cite{yi2024federated} & $85.56\pm0.31$ & $39.57\pm0.42$ & $22.48\pm0.54$
& $84.91\pm0.31$ & $\underline{57.45\pm0.56}$ & $31.95\pm0.49$ \\

\midrule
\multirow{3}{*}{\makecell[c]{Prototype\\sharing}}
& FedProto \cite{tan2022fedproto} & $81.45\pm0.51$ & $39.75\pm0.42$ & $17.51\pm0.32$
& $81.35\pm0.42$ & $52.98\pm0.69$ & $30.45\pm0.68$ \\
& FedTGP \cite{zhang2024fedtgp} & $83.69\pm0.42$ & $40.14\pm0.35$ & $23.71\pm0.42$
& $83.81\pm0.52$ & $55.04\pm0.37$ & $31.09\pm0.31$ \\
& FedRE \cite{yao2026fedre} & $85.39\pm0.37$ & $40.80\pm0.40$ & $23.53\pm0.34$
& $83.31\pm0.33$ & $54.41\pm0.35$ & $31.63\pm0.25$ \\

\midrule
Relation topology
& FedTopo & $\mathbf{87.38\pm0.26}$ & $\mathbf{43.05\pm0.17}$ & $\mathbf{26.71\pm0.25}$
& $\mathbf{86.26\pm0.19}$ & $\mathbf{59.30\pm0.18}$ & $\mathbf{34.14\pm0.22}$ \\[2pt]

\bottomrule
\end{tabular}
}
\end{table*}

Fig.~\ref{fig:mechanism}(a) presents a client-level paired comparison of the prototype, topology, and reliability-modulated inconsistencies, where each line connects the three values of the same client. The average prototype inconsistency is $0.163$, already above the topology inconsistency of $0.126$, showing that relation-level structure is more stable across clients than absolute prototype coordinates. The few clients whose topology value slightly exceeds the prototype one all have very few observed classes, whose scarce class pairs make the raw relation estimate noisier. More importantly, reliability modulation further reduces the effective inconsistency to $0.058$, with the largest gains on class-sparse clients whose few observed pairs yield the noisiest relations. This confirms that statistically weak class-pair relations receive lower influence during topology-guided local optimization. 

Fig.~\ref{fig:mechanism}(b) provides a class-level view of prototype dispersion, reporting for each class the average deviation $|\phi(p_{k,c},p_c^g)-1|$ between local prototypes and the aggregated prototype center. The consistently high values indicate that same-class prototypes learned by heterogeneous clients are widely scattered in the feature space. Fig.~\ref{fig:mechanism}(c) then examines the analogous local–global discrepancy at the relation level by visualizing $|R_k(c,c')-R^g(c,c')|$ for each class pair. Because relation topology encodes relative class relationships rather than absolute feature positions, these discrepancies concentrate in a smaller range, indicating that relation-level topology is the more stable knowledge carrier. Finally, Fig.~\ref{fig:mechanism}(d) shows the reliability-modulated discrepancy, where weakly supported class-pair entries are further attenuated by $S^g$.

Together, these observations confirm the analysis in Section~\ref{sec:theory}. Relation-level topology is more stable than absolute prototypes under coordinate mismatch, and reliability weighting complements it by selectively discounting discrepancy on weakly supported pairs.

\section{Experiments}

\subsection{Experimental Setup}
\label{sec:setup}

\subsubsection{Datasets and Non-IID Partition}
We evaluate FedTopo on three standard controlled FL benchmarks, namely CIFAR-10, CIFAR-100 \cite{krizhevsky2009learning}, and Tiny-ImageNet \cite{russakovsky2015imagenet, le2015tiny}. These datasets contain $10$, $100$, and $200$ classes, respectively. To assess robustness under data heterogeneity, we construct non-IID client datasets with two widely adopted label-skew settings. In the Dirichlet setting \cite{hsu2019measuring, yurochkin2019bayesian}, for each class $c$, the sample proportions over the $N$ clients are drawn as $(q_{c,1},\ldots,q_{c,N})\sim\mathrm{Dir}(\alpha_{\mathrm{Dir}})$,
where $q_{c,k}$ denotes the proportion of class-$c$ samples allocated to client $k$. A smaller $\alpha_{\mathrm{Dir}}$ indicates stronger label skew, and we set $\alpha_{\mathrm{Dir}}=0.1$ by default. In the pathological setting, each client holds samples from $2$, $10$, and $20$ classes on CIFAR-10, CIFAR-100, and Tiny-ImageNet, respectively. This produces local datasets with many missing classes, representing an extreme case of label skew. Fig.~\ref{fig:data} visualizes the CIFAR-10 client-class distributions under the Dirichlet ($\alpha_{\mathrm{Dir}}=0.1$) and pathological settings, where the bubble size denotes the per-class sample count on each client. 

\subsubsection{Model Heterogeneity Settings}
To simulate model heterogeneity, each client is assigned a local architecture from a heterogeneous model pool. The pool comprises eight architectures, namely CNN \cite{mcmahan2017communication}, GoogLeNet \cite{szegedy2015going}, MobileNetV2 \cite{sandler2018mobilenetv2}, and ResNet-18/34/50/101/152 \cite{he2016deep}. Architectures are assigned to clients round-robin by client index. To produce feature representations of a unified dimension $d$, we append an adaptive average pooling layer after each backbone. By default, we set \(d=512\).

\subsubsection{Compared Methods}
We compare FedTopo with representative model-heterogeneous FL methods from three major knowledge-sharing paradigms. For partial parameter-sharing methods, we include LG-FedAvg \cite{liang2020think}, FedGen \cite{zhu2021data}, FedGH \cite{yi2023fedgh}, and FedSCE \cite{zhang2025subspace}. For knowledge distillation-based methods, we include FML \cite{shen2020federated}, FedKD \cite{wu2022communication}, and FedMRL \cite{yi2024federated}. For prototype-sharing methods, we include FedProto \cite{tan2022fedproto}, FedTGP \cite{zhang2024fedtgp}, and FedRE \cite{yao2026fedre}. We additionally include Local, where each client trains its model independently without any collaboration, as a no-sharing reference that indicates the difficulty of each setting.

\begin{figure}[t]
\centering
\subfloat[CIFAR-10]{
\includegraphics[width=0.475\linewidth]{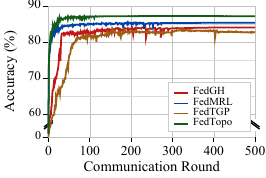}
\label{fig:cifar10-curve}
}
\subfloat[CIFAR-100]{
\includegraphics[width=0.475\linewidth]{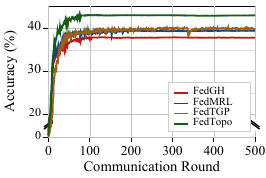}
\label{fig:cifar100-curve}
}
\caption{Convergence on CIFAR-10/100 ($\alpha_{\mathrm{Dir}}=0.1$).}
\label{fig:two_plots}
\end{figure}

\subsubsection{Implementation Details}
All experiments are implemented in PyTorch 2.8.0 with Python 3.10 and conducted on an NVIDIA RTX 3090 GPU. We use $N=20$ clients with full participation (ratio $1.0$) in each round. We run 500 communication rounds, and each participating client performs one local epoch with a batch size of $256$. All local models are trained with SGD at a learning rate of $0.05$. For FedTopo, we set the topology-aware loss weight \(\lambda=1.0\), the modulation strength \(\gamma=1.0\), the number of topology-guided negative classes \(q=3\), and the relation EMA coefficient \(\rho=0.7\), following the sensitivity analysis in Fig.~\ref{fig:param}. Topology guidance is enabled only after $3$ warm-up rounds, since the global topology is still unreliable in the earliest rounds.

\subsubsection{Evaluation Metrics}
Each client's local data is split into a local training set $75\%$ and a local test set $25\%$ drawn from the same label distribution. We evaluate each heterogeneous model on its own local test set. For each method, we report the test accuracy averaged across all clients, given as the mean and standard deviation over three independent runs.

\subsection{Overall Performance Comparison}

Table~\ref{tab:main_results} reports the overall performance under both Dirichlet and pathological non-IID settings on CIFAR-10, CIFAR-100, and Tiny-ImageNet. FedTopo achieves the best performance across all evaluated settings, obtaining $87.38\%$, $43.05\%$, and $26.71\%$ under the Dirichlet setting and $86.26\%$, $59.30\%$, and $34.14\%$ under the pathological setting on the three datasets. 

These results show that FedTopo consistently outperforms all three sharing paradigms. As analyzed in Section~\ref{sec:analysis}, their reliance on cross-client alignment weakens under heterogeneous backbones, whereas FedTopo shares only relative class relations. Under the personalized local-test protocol, the Local reference is already strong, so a useful sharing scheme must avoid negative transfer. Several baselines fail to surpass Local, whereas FedTopo improves over it across all six settings by $1.95$ to $6.01$ points, confirming consistently positive transfer. FedTopo's lead over the strongest baseline is largest on the many-class datasets under the pathological setting, reaching $+1.85$ on CIFAR-100 and $+1.34$ on Tiny-ImageNet, where data heterogeneity is most severe. 

Fig.~\ref{fig:two_plots} compares convergence on CIFAR-10 and CIFAR-100. FedTopo converges the fastest and stabilizes at the highest accuracy on both datasets with the smoothest late-training curve, whereas the prototype-based FedTGP converges more slowly and oscillates markedly. This confirms that relation topology delivers faster and more stable global guidance under model heterogeneity.

\subsection{Effect of Model and Data Heterogeneity}

We further vary the degree of model and data heterogeneity to examine whether FedTopo's advantage holds consistently across different levels.

\subsubsection{Effect of Model Heterogeneity}
We evaluate different methods under increasing model heterogeneity while keeping the data partition unchanged. The Mild setting contains CNN and ResNet-18, the Moderate setting further includes GoogLeNet and MobileNetV2, and the Strong setting follows the 8-backbone heterogeneous setup described in Section~\ref{sec:setup}. 

As shown in Fig.~\ref{fig:heterogeneity}(a), FedTopo achieves the best performance at all heterogeneity levels, decreasing only slightly from $87.60\%$ at Mild to $87.38\%$ at Strong. As backbone heterogeneity grows, FedTopo maintains stable accuracy, whereas FedMRL drops noticeably and FedTGP fluctuates more. This indicates that methods relying on shared heads, auxiliary models, or prototype-level guidance are more sensitive to backbone-induced feature-space mismatch.

In contrast, FedTopo transfers knowledge through relation-level topology rather than aggregating feature coordinates, so its supervisory signal stays transferable as backbones diverge.

\subsubsection{Effect of Data Heterogeneity}

We further evaluate robustness under varying data heterogeneity by changing the Dirichlet parameter $\alpha_{\mathrm{Dir}}$. A smaller $\alpha_{\mathrm{Dir}}$ yields a more skewed label distribution, leaving each client with fewer effective classes. Under personalized local evaluation, this makes the local task easier, so accuracy decreases monotonically as $\alpha_{\mathrm{Dir}}$ increases.

\begin{figure}[t]
\centering
\subfloat[Backbone heterogeneity.]{
\includegraphics[width=0.475\linewidth]{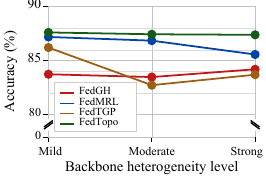}
\label{fig:backbone_het}
}
\subfloat[Data heterogeneity.]{
\includegraphics[width=0.475\linewidth]{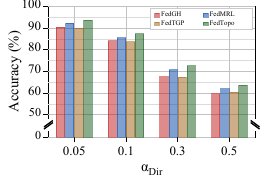}
\label{fig:data_het}
}
\caption{Effect of model and data heterogeneity on CIFAR-10.}
\label{fig:heterogeneity}
\end{figure}

As shown in Fig.~\ref{fig:heterogeneity}(b), FedTopo achieves the best performance under all Dirichlet settings. It obtains $93.74\%$, $87.38\%$, $72.59\%$, and $63.57\%$ accuracy at $\alpha_{\mathrm{Dir}}=0.05, 0.1, 0.3, 0.5$, improving over the strongest baseline by $1.45$, $1.82$, $1.76$, and $1.66$ percentage points, respectively. All methods follow the same monotonic trend, while FedTopo retains its margin at every skew level, confirming the stable benefit of relation-level topology sharing under data heterogeneity.

\subsection{Ablation Study}

We conduct an ablation study to investigate the contribution of each key component in FedTopo, with results reported in Table~\ref{tab:ablation}. All variants are evaluated on CIFAR-10 under the Dirichlet setting with $\alpha_{\mathrm{Dir}}=0.1$, using the same protocol as the complete FedTopo. Two variants isolate the role of collaboration. The \textbf{w/o Topo. Loss} variant drops the auxiliary loss entirely, reducing FedTopo to standalone local training. The \textbf{w/ Local topology} variant keeps the loss but guides each client with its own local relations $R_k$ rather than the shared global topology $(R^g,S^g)$. The other four variants each remove one internal component—relation reliability, Top-$q$ selection, the missing-aware mask, or the topology EMA. Here, \(\Delta\) denotes the accuracy drop compared with the complete FedTopo.

As shown in Table~\ref{tab:ablation}, removing any component degrades the performance. The two collaboration anchors decompose this gain. Starting from $85.43\%$ for standalone training, local topology guidance alone raises accuracy to $86.01\%$, a gain of $0.58$, and sharing the aggregated global topology lifts it to $87.38\%$, an additional $1.37$. Federated sharing thus contributes more than twice the gain of local guidance, confirming that FedTopo's improvement comes chiefly from \emph{sharing} relation topology rather than from the auxiliary loss alone. Among the internal components, dropping Top-$q$ selection causes the largest reduction ($1.54\%$), since using all negatives injects noisy feedback. Removing the missing-aware mask or relation reliability costs $1.29\%$ and $0.94\%$, indicating that locally absent or weakly supported relations should be suppressed. Disabling the topology EMA costs $1.02\%$, confirming the benefit of smoothing the topology across rounds. Together, these results validate both federated topology sharing and the reliability-aware, selective design of FedTopo.

\begin{table}[t]
\centering
\caption{Ablation study of FedTopo. \(\Delta\) denotes the accuracy drop compared with the complete FedTopo.}
\label{tab:ablation}
\begin{tabular}{lc}
\toprule
\textbf{Variant} & \textbf{Accuracy (\%) \((\Delta)\)} \\
\midrule
w/o Topo. Loss          & 85.43 \, (1.95) \\
w/ Local topology       & 86.01 \, (1.37) \\ 
w/o Reliability \(S^g\) & 86.44 \, (0.94) \\
w/o Top-\(q\)           & 85.84 \, (1.54) \\
w/o Missing Mask        & 86.09 \, (1.29) \\
w/o Topology EMA        & 86.36 \, (1.02) \\
\textbf{FedTopo}        & \textbf{87.38} \\
\bottomrule
\end{tabular}
\end{table}

\begin{figure*}[htbp]
\centering
\begin{subfigure}{0.245\textwidth}
\centering
\includegraphics[width=\linewidth]{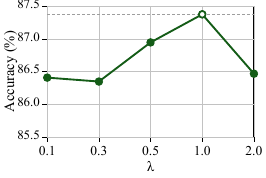}
\caption{Loss weight \(\lambda\).}
\label{fig:1a}
\end{subfigure}
\hfill
\begin{subfigure}{0.245\textwidth}
\centering
\includegraphics[width=\linewidth]{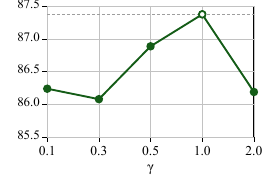}
\caption{Modulation strength \(\gamma\).}
\label{fig:1b}
\end{subfigure}
\hfill
\begin{subfigure}{0.245\textwidth}
\centering
\includegraphics[width=\linewidth]{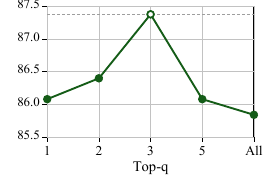}
\caption{Top-\(q\) negatives.}
\label{fig:1c}
\end{subfigure}
\hfill
\begin{subfigure}{0.245\textwidth}
\centering
\includegraphics[width=\linewidth]{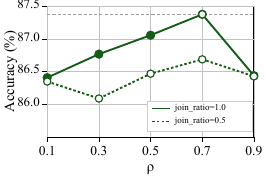}
\caption{EMA coefficient \(\rho\).}
\label{fig:1d}
\end{subfigure}
\caption{Parameter sensitivity analysis of FedTopo. The hollow marker indicates the best-performing setting.}
\label{fig:param}
\end{figure*}

\begin{table*}[t]
\caption{Comparison of communication and extra computation overhead on CIFAR-10.}
\label{table_overhead_comparison}
\centering
\begin{tabular}{lccccc}
\hline
\textbf{Method} & \textbf{Upload Params.} & \textbf{Download Params.} & \textbf{Server Train} & \textbf{Train FLOPs} & \textbf{Infer FLOPs} \\
\hline
FedGH   & 5,120   & 5,130   & \checkmark & 40.69G / extra batch & 0 \\
FedMRL  & 878,538 & 878,538 & $\times$   & 4.30G / batch        & 4.30G / batch \\
FedTGP  & 5,120   & 5,120   & \checkmark & 0.39M / batch        & 3.93M / batch \\
\textbf{FedTopo} & \textbf{110} & \textbf{200} & $\times$ & \textbf{0.013M / batch} & \textbf{0} \\
\hline
\multicolumn{6}{l}{\footnotesize \emph{Note:} For FedGH, ``extra batch'' denotes one additional feature-extraction batch for prototype collection after local training.}
\end{tabular}
\end{table*}

\subsection{Parameter Sensitivity Analysis}

We analyze FedTopo's sensitivity to four key hyperparameters, namely the topology loss weight \(\lambda\), the modulation strength \(\gamma\), the number of selected topology-similar negative classes \(q\), and the EMA coefficient \(\rho\). We vary one hyperparameter at a time, fixing the others to their default values. 

\subsubsection{Effect of the Topology Loss Weight \(\lambda\)}
The parameter \(\lambda\) controls the contribution of the topology-aware auxiliary loss. As shown in Fig.~\ref{fig:param}(a), FedTopo achieves the best accuracy when \(\lambda=1.0\). When \(\lambda\) is small, the topology-guided supervision becomes weak and brings limited improvement. When \(\lambda\) is too large, the auxiliary objective may overemphasize topology-similar negative competition and interfere with the standard classification loss. A moderate weight therefore best balances local classification against topology-guided discrimination.

\subsubsection{Effect of the Modulation Strength \(\gamma\)}
The parameter \(\gamma\) determines the strength of logit modulation for topology-similar negative classes. Fig.~\ref{fig:param}(b) shows that \(\gamma=1.0\) provides the best performance. A small \(\gamma\) only slightly increases the logits of topology-similar negatives, leading to insufficient hard negative competition. In contrast, an overly large $\gamma$ over-amplifies these negatives and disturbs local optimization. A moderate $\gamma$ therefore strengthens topology-similar negatives just enough without overwhelming the standard objective.

\subsubsection{Effect of the Top-\(q\) Negative Selection}
The parameter \(q\) controls how many topology-similar negative classes are selected for logit enhancement. As shown in Fig.~\ref{fig:param}(c), the best performance is achieved when \(q=3\). A very small \(q\) provides a narrow topology signal, since only a few negatives are emphasized. Conversely, selecting too many negatives brings in weakly related classes and dilutes the truly confusing ones. The drop under the \textsf{All} setting confirms this, since using every negative injects noisy feedback. A small, selective set of topology-similar negatives is therefore preferable to using all of them.

\subsubsection{Effect of the EMA Coefficient \(\rho\)}
The parameter \(\rho\) controls the EMA smoothing of the global topology across communication rounds. Fig.~\ref{fig:param}(d) shows that \(\rho=0.7\) achieves the best accuracy. A small \(\rho\) makes the global topology rely heavily on the current round, leaving it sensitive to client sampling and local class-pair fluctuations. An overly large \(\rho\) preserves too much historical topology and slows down the adaptation to newly aggregated semantic relations. Thus, a balanced EMA coefficient improves both the stability and adaptability of global topology estimation. Moreover, $\rho=0.7$ remains optimal under a lower join ratio of $0.5$, showing that this EMA setting is robust to partial client participation, even though the overall accuracy decreases as fewer clients contribute each round.

Overall, FedTopo performs best with moderate topology guidance and selective negative enhancement. We adopt $\lambda=1.0$, $\gamma=1.0$, $q=3$, and $\rho=0.7$ as the default configuration. These results confirm that each hyperparameter contributes to stable and effective topology-guided training.

\subsection{Communication and Computation Cost}

We analyze the communication and computation cost on CIFAR-10. Table~\ref{table_overhead_comparison} reports the upload and download parameters per client per round and the extra FLOPs beyond the standard local forward and backward pass.

FedTopo only uploads the local relation topology and class statistics, and downloads the global relation topology and reliability matrix. With \(C=10\), this corresponds to only 110 uploaded parameters and 200 downloaded parameters. For larger label spaces the cost grows as $O(C^2)$; even on Tiny-ImageNet ($C=200$) it remains below the $C\times d$ cost of prototype methods. In contrast, methods that upload class-wise vectors, such as FedGH and FedTGP, transmit $C\times d = 5{,}120$ parameters per round when $d=512$, while FedMRL communicates an auxiliary global model of $878{,}538$ parameters. This makes the communication cost of FedTopo substantially smaller in this setting.

For computation cost, FedTopo only introduces a topology-guided auxiliary loss during local training. Since the topology modulation operates on class-level logits rather than high-dimensional features or model parameters, its extra training cost is only 0.013M FLOPs per batch. More importantly, FedTopo introduces no inference overhead, because the global topology is used only during training and each client directly uses its local model for testing. By comparison, FedTGP requires prototype-distance computation during inference, and FedMRL requires an auxiliary global model and projection module during both training and inference. These results show that FedTopo guides training efficiently while adding no inference cost at deployment.

\section{Conclusion}
In this work, we propose FedTopo, a model-heterogeneous FL framework for global knowledge sharing under heterogeneous feature spaces and non-IID data distributions. Different from methods that share model parameters, logits, or prototype coordinates, FedTopo communicates relation-level class topology, which provides a more transferable form of global knowledge across heterogeneous client backbones. Specifically, FedTopo constructs local relation topologies on the clients, performs reliability- and missing-aware aggregation on the server, and uses the global topology to guide local training through topology-guided negative enhancement. Since the topology guidance is only used during training, FedTopo introduces no extra inference overhead. Extensive experiments demonstrate that FedTopo consistently outperforms baselines across different benchmarks, backbone heterogeneity levels, and data partitions, while maintaining low communication and computation cost. A limitation is that the communication cost grows quadratically with the number of classes. Future work will explore extending topology-based knowledge sharing to multi-modal and open-set federated learning scenarios.

\section*{Acknowledgment}
Supported by the Fundamental Research Funds for the Central Universities (No. 2025YJS042) and the National Natural Science Foundation of China (No. 62576027).

\bibliographystyle{IEEEtran}
\bibliography{ref}

\end{document}